\documentclass[letterpaper]{article}
\usepackage{aaai}
\usepackage{times}
\usepackage{amsfonts}
\usepackage{amsmath}
\usepackage{amsthm}
\usepackage{graphicx}
\usepackage{epstopdf}
\usepackage{subfigure}
\usepackage[ruled,vlined,linesnumbered]{algorithm2e}
\providecommand{\SetAlgoLined}{\SetLine}

\usepackage{algpseudocode}
\usepackage{helvet}
\usepackage{courier}
\graphicspath{{./figs/}}
\begin{document}

\newtheorem{theorem}{Theorem}
\newtheorem{lemma}{Lemma}

\renewcommand{\algorithmicrequire}{\textbf{Input:}}
\renewcommand{\algorithmicensure}{\textbf{Output:}}

\newcommand{\eg}{\emph{e.g.}\xspace}
\newcommand{\st}{\emph{s.t.}\xspace}
\newcommand{\ie}{\emph{i.e.}\xspace}
\newcommand{\etc}{\emph{etc.}\xspace}
\newcommand{\wrt}{\emph{w.r.t.}\xspace}
\newcommand{\etal}{\emph{et al.}\xspace}
%
\title{A Convex Formulation for Spectral Shrunk Clustering}
\author{Xiaojun Chang$^1$, Feiping Nie$^{2,3}$, Zhigang Ma$^4$, Yi Yang$^1$ and Xiaofang Zhou$^5$ \\
$^1$Centre for Quantum Computation and Intelligent Systems, University of Technology Sydney, Australia. \\
$^2$Center for OPTical IMagery Analysis and Learning, Northwestern Polytechnical University, Shaanxi, China. \\
$^3$Department of Computer Science and Engineering, University of Texas at Arlington, USA. \\
$^4$School of Computer Science, Carnegie Mellon University, USA.\\
$^5$School of Information Technology \& Electrical Engineering, The University of Queensland, Australia.\\
cxj273@gmail.com,feipingnie@gmail.com,kevinma@cs.cmu.edu,yiyang@cs.cmu.edu,zxf@itee.uq.edu.au.\\
}
\maketitle
\begin{abstract}
\begin{quote}
Spectral clustering is a fundamental technique in the field of data mining and information processing. Most existing spectral clustering algorithms integrate dimensionality reduction into the clustering process assisted by manifold learning in the original space. However, the manifold in reduced-dimensional subspace is likely to exhibit altered properties in contrast with the original space. Thus, applying manifold information obtained from the original space to the clustering process in a low-dimensional subspace is prone to inferior performance. Aiming to address this issue, we propose a novel convex algorithm that mines the manifold structure in the low-dimensional subspace. In addition, our unified learning process makes the manifold learning particularly tailored for the clustering. Compared with other related methods, the proposed algorithm results in more structured clustering result. To validate the efficacy of the proposed algorithm, we perform extensive experiments on several benchmark datasets in comparison with some state-of-the-art clustering approaches. The experimental results demonstrate that the proposed algorithm has quite promising clustering performance.
\end{quote}
\end{abstract}

\section{Introduction}

Clustering has been widely used in many real-world applications \cite{cluster,DBLP:conf/pkdd/WangNH14}. The objective of clustering is to cluster the original data points into various clusters, so that data points within the same cluster are dense while those in different clusters are far away from each other \cite{surveyspectral}. Researchers have proposed a variety of clustering algorithms, such as $K$-means clustering and mixture models \cite{DBLP:conf/miccai/WangWNYCSSH14,kddclustering,DBLP:journals/tnn/NieZTXZ11}, \emph{etc}.

The existing clustering algorithms, however, mostly work well when the samples' dimensionality is low. When partitioning high-dimensional data, the performance of these algorithms is not guaranteed. For example, $K$-means clustering iteratively assigns each data point to the cluster with the closest center based on specific distance/similarity measurement and updates the center of each cluster. But the distance/similarity measurements may be inaccurate on high-dimensional data, which tends to limit the clustering performance. As suggested by some researchers, many high-dimensional data may exhibit dense grouping in a low-dimensional subspace~\cite{SEC}. Hence, researchers have proposed to first project the original data into a low-dimensional subspace via some dimensionality reduction techniques and then cluster the computed low-dimensional embedding for high-dimensional data clustering. For instance, a popular approach is to use Principle component analysis (PCA) to reduce the dimensionality of the original data followed by Kmeans for clustering (PcaKm) \cite{surcs}. Ding et al. present a clustering algorithm based on Linear discriminant analysis (LDA) method \cite{DBLP:conf/icml/DingL07}. Ye et al propose discriminative $K$-means (DisKmeans) clustering which unifies the iterative procedure of dimensionality reduction and $K$-means clustering into a trace maximization problem \cite{dkmeans}.

Another genre of clustering, \emph{i.e.}, spectral clustering \cite{normalizedcut} integrates dimensionality reduction into its clustering process. The basic idea of spectral clustering is to find a clustering assignment of the data points by adopting the spectrum of similarity matrix that leverages the nonlinear manifold structure of original data. Spectral clustering has been shown to be easy to implement and oftentimes it outperforms traditional clustering methods because of its capacity of mining intrinsic geometric structures, which facilitates partitioning data with more complicated structures. The benefit of utilizing manifold information has been demonstrated in many applications, such as image segmentation and web mining. Due to the advantage of spectral clustering, different variants of spectral clustering algorithms have been proposed these years \cite{aaai15MVSCBGLi}. For example, local learning-based clustering (LLC) \cite{LLC} utilizes a kernel regression model for label prediction based on the assumption that the class label of a data point can be determined by its neighbors. Self-tuning SC \cite{selftuning} is able to tune parameters automatically in an unsupervised scenario. Normalized cuts is capable of balancing the volume of clusters for the usage of data density information \cite{normalizedcut}.

Spectral clustering is essentially a two-stage approach, \emph{i.e.}, manifold learning based in the original high-dimensional space and dimensionality reduction. To achieve proper clustering, spectral clustering assumes that two nearby data points in the high density region of the reduced-dimensional space have the same cluster label. However, this assumption does not always hold. More possibly, these nearest neighbors may be far away from each other in the original high-dimensional space due to the curse of dimensionality. That being said, the distance measurement of the original data could not precisely reflect the low-dimensional manifold structure, thus leading to suboptimal clustering performance.

Intuitively, if the manifold structure in the low-dimensional space is precisely captured, the clustering performance could be enhanced when applied to high-dimensional data clustering. Aiming to achieve this goal, we propose a novel clustering algorithm that is able to mine the inherent manifold structure of the low-dimensional space for clustering. Moreover, compared to traditional spectral clustering algorithms, the shrunk pattern learned by the proposed algorithm does not have an orthogonal constraint, giving it more flexibility to fit the manifold structure. It is worthwhile to highlight the following merits of our work:

\begin{itemize}
\item The proposed algorithm is more capable of uncovering the manifold structure. Particularly, the shrunk pattern does not have the orthogonal constraint, making it more flexible to fit the manifold structure. 
\item The integration of manifold learning and clustering makes the former particularly tailored for the latter. This is intrinsically different from most state-of-the-art clustering algorithms.
\item The proposed algorithm is convex and converges to global optimum, which indicates that the proposed algorithm does not rely on the initialization.
\end{itemize}

The rest of this paper is organized as follows. After reviewing related work on spectral clustering in section 2, we detail the proposed algorithm in section 3. Extensive experimental results are given in section 4 and section 5 concludes this paper.

\section{Related Work}
Our work is inspired by spectral clustering. Therefore, we review the related work on spectral clustering in this section.
\subsection{Basics of Spectral Clustering}
To facilitate the presentation, we first summarize the notations that will be frequently used in this paper. Given a dataset $\mathcal{X} = \{x_1, \dots, x_n \}$, $x_i \in \mathbb{R}^d (1 \leq i \leq n)$ is the $i$-th datum and $n$ is the total number of data points. The objective of clustering is to partition $\chi$ into $c$ clusters. Denote the cluster assignment matrix by $Y = \{y_1, \dots, y_n\} \in \mathbb{R}^{n \times c}$, where $y_i \in {\{0, 1\}}^{c \times 1}~(1 \leq i \leq n)$ is the cluster indicator vector for the datum $x_i$. The $j$-th element of $y_i$ is 1 if $x_i$ is clustered to the $j$-th cluster, and 0 otherwise.

Existing spectral clustering algorithms adopt a weighted graph to partition the data. Let us denote $\mathcal{G} = \{ \mathcal{X}, A \}$ as a weighted graph with a vertex set $\mathcal{X}$ and an affinity matrix $A \in \mathbb{R}^{n \times n}$. $A_{ij}$ is the affinity of a pair of vertexes of the weighted graph. $A_{ij}$ is commonly defined as:

\begin{eqnarray}\nonumber
\small
A_{ij} =
\left\{
\begin{array}{lll}
exp(-\frac{\|x_i - x_j\|^2}{\delta ^2}), \textit{if $x_i$ and $x_j$ are $k$ nearest neighbors.} \\
0, ~~~~~~~~~~~~~~~~~~~~~~~~~\textit{otherwise.}
\end{array}
\right.
\end{eqnarray}
where $\delta$ is the parameter to control the spread of neighbors. The Laplacian matrix $L$ is computed according to $L = D - A$, where $D$ is a diagonal matrix with the diagonal elements as $D_{ii} = \sum_j A_{ij}, \forall i$. Following the work in \cite{dkmeans}, we denote the scaled cluster indicator matrix $F$ as follows:

\begin{equation}
F = [F_1, F_2, \dots, F_n]^T = Y(Y^TY)^{-\frac{1}{2}},
\end{equation}
where $F_i$ is the scaled cluster indicator of $x_i$. The $j$-th column of $F$ is defined as follows by \cite{dkmeans}:

\begin{equation}
f_j = \left[\underbrace{0,\dots , 0,}_{\sum_{i=1}^{j-1}n_i} \underbrace{\frac{1}{\sqrt{n_j}}, \dots , \frac{1}{\sqrt{n_j}},}_{n_j} \underbrace{0, \dots , 0}_{\sum_{i=j+1}^c n_k} \right],
\end{equation}
which indicates which data points are partitioned into the $j$-th cluster $C_j$. $n_j$ is the number of data points in cluster $C_j$.

The objective function of spectral clustering algorithm is generally formulated as follows:
\begin{equation}
\begin{aligned}
\min_F~ & Tr(F^TLF) \\
s.t. ~& F = Y(Y^TY)^{-\frac{1}{2}}
\end{aligned}
\label{objsc}
\end{equation}
where $Tr(\cdot)$ denotes the trace operator. By denoting $I$ as an identity matrix, we can define the normalized Laplacian matrix $L_n$ as:

\begin{equation}
L_n = I - D^{-\frac{1}{2}}AD^{-\frac{1}{2}}.
\end{equation}

By replacing $L$ in Eq. \eqref{objsc} with the normalized Laplacian matrix, the objective function becomes the well-known SC algorithm normalized cut \cite{normalizedcut}. In the same manner, if we replace $L$ in Eq. \eqref{objsc} by the Laplacian matrix obtained by local learning \cite{ldmg}\cite{LLC}, the objective function converts to Local Learning Clustering (LLC).

\subsection{Progress on Spectral Clustering}
Being easy to implement and promising for many applications, spectral clustering has been widely studied for different problems. Chen et al. propose a Landmark-based Spectral Clustering (LSC) for large scale clustering problems~\cite{DBLP:conf/aaai/ChenC11}. Specifically, a few representative data points are first selected as the landmarks and the original data points are then represented as the linear combinations of these landmarks. The spectral clustering is performed on the landmark-based representation. Yang et al. propose to utilize a nonnegative constraint to relax the elements of cluster indicator matrix for spectral clustering~\cite{DBLP:conf/aaai/YangSNJZ11}. Liu et al. propose to compress the original graph used for spectral clustering into a sparse bipartite graph. The clustering is then performed on the bipartite graph instead, which improved the efficiency for large-scale data~\cite{DBLP:conf/ijcai/LiuWDH13}. Xia et al. propose a multi-view spectral clustering method based on low-rank and sparse decomposition~\cite{DBLP:conf/aaai/XiaPDY14}. Yang et al. propose to use Laplacian Regularized L1-Graph for clustering~\cite{DBLP:conf/aaai/YangWYWCH14}. Tian et al. recently propose to adopt deep learning in spectral clustering~\cite{DBLP:conf/aaai/TianGCCL14}.

In spite of the encouraging progress, few of the existing spectral clustering methods have considered learn the manifold in the low-dimensional subspace more precisely, not to mention integrating such manifold learning and clustering into a unified framework. This issue shall be addressed in this paper for boosted clustering performance.

\section{The Proposed Algorithm}

In this section, we present the details of the proposed algorithm. A fast iterative method is also proposed to solve the objective function.

\subsection{Problem Formulation}
Our algorithms is built atop the aim of uncovering the utmost manifold structure in the low-dimensional subspace of original data. Inspired by~\cite{DBLP:journals/ipm/HouNJZW13}, we adopt the pattern shrinking during the manifold learning and the shrunk patterns are exploited for clustering simultaneously.

To begin with, we have the following notations. Denote the shrunk patterns of $n$ data samples as $\{g_1, \cdots, g_n \}$, where $g_i \in \mathbb{R}^c$. We first obtain spectral embedding $F$ of the original samples by minimizing the traditional spectral clustering algorithm $\min Tr(F^TL_nF)$, where $L_n$ is a normalized Laplacian matrix.

Next, the shrunk patterns are computed by satisfying the following requirements. (1) The shrunk patterns should keep consistency with the spectral embedding. To be more specific, the shrunk patterns should not be far away from the spectral clustering. (2) Note that nearby points are more likely to belong to the same cluster. We thus design a similarity matrix to measure pair similarity of any two spectral embedding, which the shrunk patters should follow.

To characterize the manifold structure of the spectral embedding $\{f_1, \cdots , f_n \}$, a $k$-nearest neighbor graph is constructed by connecting each point to its $k$ nearest neighbors. The similarity matrix, $W$, is computed by $W_{ij} = exp(-\frac{\| f_i - f_j\|^2}{\delta ^2})$.

From this similarity matrix, we can observe that if two spectral embeddings are nearby, they should belong to the same cluster and the corresponding weight should be large, which satisfies the first requirement \cite{DBLP:conf/iccv/NieWHD11}.

To keep the local similarity of spectral embedding, we propose to optimize the following objective function.

\begin{equation}
\min_G \sum_{ij} W_{ij} \| g_i - g_j \|_2
\end{equation}

We also aim to keep the consistency between spectral embedding and shrunk patterns. Hence, we propose to  minimize the following loss function directly.

\begin{equation}
\min_G \|G - F\|_2^2
\end{equation}

To this end, we formulate the objection function as follows:

\begin{equation}
\min_G \|G - F\|_2^2 + \gamma \sum_{i, j} W_{ij} \| g_i - g_j \|_2
\end{equation}
where $\gamma$ is a balance parameter.

It can be easily proved that our formulation is convex. Due to the space limit, we omit the proof here. Since our method exploits shrunk patterns as the input for clustering, we name it Spectral Shrunk Clustering (SSC).

As indicated in \cite{mazhigang1,DBLP:conf/cikm/KongDH11}, the least square loss function is not robust to outliers. To make our method even more effective, we follow \cite{mazhigang1,l21norm} and employ $l_{2,1}$-norm to handle the outliers. The objective function is rewritten as follows:

\begin{equation}
\min_G \|G - F\|_{2, 1} + \gamma \sum_{i, j} W_{ij} \| g_i - g_j \|_2
\label{finalsc}
\end{equation}

\subsection{Optimization}

The proposed function involves the $l_{2,1}$-norm, which is difficult to solve in a closed form. We propose to solve this problem in the following steps. Denote $H = G - F$ and $H = [h^1, \cdots, h^d]$, where $d$ is the dimension of spectral embedding. The objective function can be rewritten as follows:

\begin{equation}
\min_G Tr((G-F)^TS(G-F)) + \gamma \sum_{ij} w_{ij} \|g_i - g_j \|_2
\end{equation}
where
\begin{equation}
\label{Sval}
S = \begin{bmatrix}
\frac{1}{2\|h^1\|_2} & & \\
    &  \ddots  & \\
    &     & \frac{1}{2\|h^d\|_2}
\end{bmatrix}.
\end{equation}

Denote a Laplacian matrix $\widetilde{L} = \widetilde{D} - \widetilde{W}$, where $\widetilde{W}$ is a re-weighted weight matrix defined by

\begin{equation}
\widetilde{W}_{ij} = \frac{W_{ij}}{2\|g_i - g_j\|_2}
\label{rew}
\end{equation}
$\widetilde{D}$ is a diagonal matrix with the $i$-th diagonal element as $\sum_j \widetilde{W}_{ij}$.

By simple mathematical deduction, the objective function arrives at:

\begin{equation}
\min_G Tr((G-F)^TS(G-F)) + \gamma Tr(G^T\widetilde{L}G).
\label{derivative}
\end{equation}

By setting the derivative of Eq. \eqref{derivative} to $G$ to 0, we have:

\begin{equation}
\label{cri}
G = (S + \gamma \widetilde{L})^{-1}SF.
\end{equation}

Based on the above mathematical deduction, we propose an iterative algorithm to optimize the objective function in Eq. \eqref{objsc}, which is summarized in Algorithm 1. Once the shrunk patterns $G$ are obtained, we perform $K$-means clustering on it to get the final clustering result.

\begin{algorithm}
\label{alg:1}
\caption{Optimization Algorithm for SSC}
 \SetAlgoLined
 \KwData{Data $X \in \mathbb{R}^{d \times n} $, Parameter $\gamma$ and the number of clusters $c$}
 \KwResult{\\
 ~~~~~~~~~The discrete cluster assignment $Y \in \mathbb{R}^{n \times c}$}
 Compute the normalized Laplacian matrix $L_n$ \;
 Obtain the spectral embedding $F$ by using the traditional spectral clustering \;
 Compute the similarity matrix $W$ using the spectral embedding $F$ \;
 Obtain the Laplacian matrix with the reweighted weight matrix according to Eq. \eqref{rew} \;
 Set $t=0$ \;
 Initialize $G_0 \in \mathbb{R}^{n \times c}$\;
  \Repeat{Convergence}{
  Compute $H_t = G_t - F$ \;
  Compute the diagonal matrix $S_t$ according to \eqref{Sval}  \;
  Compute $G_{t+1}$ according to $G_{t+1} = (S_t + \gamma \widetilde{W})^{-1}S_tX$ \;
  t = t + 1 \;
  }
  Based on $G^*$, compute the discrete cluster assignment matrix $Y$ by using $K$-means clustering\;
  Return the discrete cluster assignment matrix $Y$.
\end{algorithm}

\subsection{Convergence Analysis}
To prove the convergence of the Algorithm \ref{alg:1}, we need the following lemma \cite{l21norm}.

\begin{lemma}
\label{lemma1}
For any nonzero vectors $g, g_t~\in \mathbb{R}^c$, the following inequality holds:
\begin{equation}
\|g\|_2 - \|g\|_2^2/2\|g_t\|_2 \leq \|g_t\|_2 - \|g_t\|_2^2/2\|g_t\|_2
\end{equation}
\end{lemma}

The following theorem guarantees that the problem in Eq. \eqref{finalsc}converges to the global optimum by Algorithm~\ref{alg:1} .

\begin{theorem}
The Algorithm \ref{alg:1} monotonically decreases the objective function value of the problem in Eq. \eqref{finalsc} in each iteration, thus making it converge to the global optimum.
\end{theorem}

\begin{proof}
Define $f(G) = Tr((G - F)^TS(G-F)$. According to Algorithm \ref{alg:1}, we know that

\begin{equation}
G_{t+1} = \arg \min_{G} f(G) + \gamma \sum_{i,j} (\widetilde{W})_{ij} \|g_i - g_j\|_2^2
\end{equation}

Note that $(\widetilde{W_t})_{ij} = \frac{W_{ij}}{2\|g_i^t - g_j^t\|_2}$, so we have

\begin{equation}
\label{sum1}
\begin{aligned}
& f(G_{t+1}) + \gamma \sum_{ij} \frac{W_{ij}\|g_i^{t+1} - g_j^{t+1}\|_2^2}{2\|g_i^t - g_j^t\|_2} \\
\leq & f(G_{t}) + \gamma \sum_{ij} \frac{W_{ij}\|g_i^{t} - g_j^{t}\|_2^2}{2\|g_i^t - g_j^t\|_2}
\end{aligned}
\end{equation}

According to Lemma \ref{lemma1}, we have

\begin{equation}
\label{sum2}
\begin{aligned}
& \sum_{ij} W_{ij}(\|g_i^{t+1} - g_j^{t+1}\|_2 - \frac{\|g_i^{t+1} - g_j^{t+1}\|_2^2}{2\|g_i^t - g_j^t\|_2}) \\
\leq & \sum_{ij} W_{ij}(\|g_i^{t} - g_j^{t}\|_2 - \frac{\|g_i^{t} - g_j^{t}\|_2^2}{2\|g_i^t - g_j^t\|_2})
\end{aligned}
\end{equation}

By summing Eq. \eqref{sum1} and Eq. \eqref{sum2}, we arrive at:

\begin{equation}
\begin{aligned}
& f(G_{t+1}) + \gamma \sum_{ij} W_{ij}\|g_i^{t+1} - g_j^{t+1}\|_2 \\
\leq & f(G_{t}) + \gamma \sum_{ij} W_{ij}\|g_i^{t} - g_j^{t}\|_2
\end{aligned}
\end{equation}

Thus, Algorithm \ref{alg:1} monotonically decreases the objective function value of the problem in Eq. \eqref{finalsc} in each iteration $t$. When converged, $G_t$ and $\widetilde{L}_t$ satisfy Eq. \eqref{cri}. As the problem in Eq. \eqref{finalsc} is convex, satisfying Eq. \eqref{cri} indicates that $G_t$ is the global optimum solution of the problem in Eq. \eqref{finalsc}. Therefore, using Algorithm 1 makes the problem in Eq. \eqref{finalsc} converge to the global optimum.

\end{proof}

\section{Experiment}
In this section, we perform extensive experiments on a variety of applications to test the performance of our method SSC. We compare SSC to several clustering algorithms including the classical $K$-means, the classical spectral clustering (SC), PCA\_Kmeans \cite{surcs}, PCA spectral clustering (PCA\_SC), LDA\_Kmeans \cite{DBLP:conf/icml/DingL07}, LDA spectral clustering (LDA\_SC), Local Learning Clustering (LLC) \cite{LLC} and SPLS \cite{DBLP:journals/ipm/HouNJZW13}.

\subsection{Datasets}

A variety of datasets are used in our experiments which are described as follows. The AR dataset \cite{AR} contains 840 faces of 120 different people. We utilize the pixel value as the feature representations. The JAFFE dataset \cite{JAFFE} consists of 213 images of different facial expressions from 10 different Japanese female models. The images are resized to $26 \times 26$ and represented by pixel values. The ORL dataset \cite{ORL} consists of 40 different subjects with 10 images each. We also resize each image to $32 \times 32$ and use pixel values to represent the images. The UMIST face dataset \cite{UMIST} consists of 564 images of 20 individuals with mixed race, gender and appearance. Each individual is shown in a range of poses from profile to frontal views. The pixel value is used as the feature representation. The BinAlpha dataset contains 26 binary hand-written alphabets and we randomly select 30 images for every alphabet. The MSRA50 dataset contains 1799 images from 12 different classes. We resize each image to $32 \times 32$ and use the pixel values as the features. The YaleB dataset \cite{yaleb} contains 2414 near frontal images from 38 persons under different illuminations. Each image is resized to $32 \times 32$ and the pixel value is used as feature representation. We additionally use the USPS dataset to validate the performance on handwritten digit recognition. The dataset consists of 9298 gray-scale handwritten digit images. We resize the images to $16 \times 16$ and use pixel values as the features.

\subsection{Setup}

The size of neighborhood, $k$ is set to 5 for all the spectral clustering algorithms. For parameters in all the comparison algorithms, we tune them in the range of $\{10^{-6}, 10^{-3}, 10^0, 10^3, 10^6 \}$ and report the best results. Note that the results of all the clustering algorithms vary on different initialization. To reduce the influence of statistical variation, we repeat each clustering 50 times with random initialization and report the results corresponding to the best objective function values. For all the dimensionality reduction based K-means and Spectral clustering, we project the original data into a low dimensional subspace of 10 to 150 and report the best results.

\subsection{Evaluation Metrics}

Following most work on clustering, we use clustering accuracy (ACC) and normalized mutual information (NMI) as our evaluation metrics in our experiments.

Let $q_i$ represent the clustering label result from a clustering algorithm and $p_i$ represent the corresponding ground truth label of an arbitrary data point $x_i$. Then $ACC$ is defined as follows:

\begin{equation}
ACC = \frac{\sum_{i=1}^n \delta (p_i, map(q_i))}{ n },
\end{equation}
where $\delta(x, y) = 1$ if $x=y$ and $\delta (x, y) = 0$ otherwise. $map(q_i)$ is the best mapping function that permutes clustering labels to match the ground truth labels using the Kuhn-Munkres algorithm. A larger ACC indicates better clustering performance.

For any two arbitrary variables $P$ and $Q$, NMI is defined as follows \cite{NMI}:

\begin{equation}
NMI = \frac{I(P, Q)}{\sqrt{H(P)H(Q)}},
\end{equation}
where $I(P, Q)$ computes the mutual information between $P$ and $Q$, and $H(P)$ and $H(Q)$ are the entropies of $P$ and $Q$. Let $t_l$ represent the number of data in the cluster $\mathcal{C}_l(1 \leq l \leq c)$ generated by a clustering algorithm and $\widetilde{t_h}$ represent the number of data points from the $h$-th ground truth class. NMI metric is then computed as follows \cite{NMI}:

\begin{equation}
NMI = \frac{\sum_{l=1}^c \sum_{h=1}^c t_{l,h} log(\frac{n \times t_{l, h}}{•t_l\widetilde{t_h}})}{\sqrt{(\sum_{l=1}^c t_l \log \frac{t_l}{n})(\sum_{h=1}^c \widetilde{t_h} \log \frac{\widetilde{t_h}}{n})}},
\end{equation}
where $t_{l,h}$ is the number of data samples that lie in the intersection between $\mathcal{C}_l$ and $h$-th ground truth class. Similarly, a larger NMI indicates better clustering performance.

\begin{table*}
\caption{Performance comparison (ACC\%$\pm$Standard Deviation) between $K$-means, Spectral Clustering, PCA\_Kmeans, LDA\_Kmeans, PCA\_SC, LDA\_SC, LLC, SPLS and SSC.}
\centering
\small
\begin{tabular}{|c||c|c|c|c|c|c|c|c|}

\hline

  &  AR &  JAFFE &  ORL &  UMIST &  binalpha & MSRA50 & YaleB & USPS \\

\hline \hline

$K$-means & $36.3 \pm 1.4$ & $75.6 \pm 1.8$ & $60.5 \pm 1.8$ & $41.3 \pm 1.6$ & $41.7 \pm 1.1$ & $46.2 \pm 1.7$ & $14.4 \pm 1.5$ & $65.4 \pm 1.7$ \\

\hline

SC & $41.6 \pm 2.1$ & $76.1 \pm 1.6$ & $72.7 \pm 2.3$ & $52.2 \pm 1.4$ & $43.6 \pm 1.5$ & $52.3 \pm 1.8$ & $34.8 \pm 1.4$ & $64.3 \pm 1.4$ \\

\hline

PCA\_Kmeans & $39.8 \pm 1.8$ & $75.8 \pm 1.5$ & $64.5 \pm 2.3$ & $48.8 \pm 1.7$ & $42.4 \pm 1.5$ & $56.0 \pm 1.9$ & $24.9 \pm 1.8$ & $69.4 \pm 1.8$\\

\hline

PCA\_SC & $43.2 \pm 1.7$ & $76.9 \pm 1.8$ & $67.8 \pm 2.1$ & $54.1 \pm 1.9$ & $44.3 \pm 1.8$ & $54.9 \pm 1.6$ & $36.8 \pm 2.0$ & $69.1 \pm 1.5$ \\

\hline

LDA\_Kmeans & $40.4 \pm 1.5$ & $76.5 \pm 1.7$ & $65.6 \pm 2.6$ & $49.7 \pm 1.8$ & $42.9 \pm 1.7$ & $56.4 \pm 1.8$ & $26.1 \pm 1.8$ & $70.1 \pm 1.3$ \\

\hline

LDA\_SC & $44.5 \pm 1.3$ & $77.4 \pm 1.9$ & $68.3 \pm 2.4$ & $54.7 \pm 1.5$ & $45.1 \pm 1.4$ & $55.1 \pm 1.7$ & $38.2 \pm 1.6$ & $70.4 \pm 1.5$ \\

\hline

LLC & $48.7 \pm 1.6$ & $78.6 \pm 1.5$ & $71.5 \pm 2.2$ & $63.3 \pm 1.8$ & $40.7 \pm 1.8$ & $48.1 \pm 1.4$ & $38.2 \pm 1.5$ & $63.9 \pm 1.7$ \\

\hline

SPLS & $49.2 \pm 1.4$ & $79.5 \pm 2.1$ & $74.2 \pm 1.8$ & $70.4 \pm 1.6$ & $48.5 \pm 1.9$ & $60.3 \pm 1.3$ & $47.3 \pm 1.7$ & $71.4 \pm 1.6$ \\

\hline

SSC & $\mathbf{51.3 \pm 1.5}$ & $\mathbf{81.2 \pm 1.6}$ & $\mathbf{76.0 \pm 1.6}$ & $\mathbf{71.1 \pm 1.8}$ & $\mathbf{49.4 \pm 1.3}$ & $\mathbf{63.2 \pm 1.1}$ & $\mathbf{49.8 \pm 1.6}$ & $\mathbf{75.5 \pm 1.9}$ \\

\hline

\end{tabular}

\label{ACC}

\end{table*}
\begin{table*}
\caption{Performance Comparison (NMI\%$\pm$Standard Deviation) between $K$-means, Spectral Clustering, PCA\_Kmeans, LDA\_Kmeans, PCA\_SC, LDA\_SC, LLC, SPLS and SSC.}

\centering

\small

\begin{tabular}{|c||c|c|c|c|c|c|c|c|}

\hline

  &  AR &  JAFFE &  ORL &  UMIST &  binalpha & MSRA50 & YaleB & USPS \\

\hline \hline

$K$-means & $68.7 \pm 3.0$ & $79.4 \pm 0.8$ & $80.3 \pm 1.8$ & $64.4 \pm 1.5$ & $58.6 \pm 1.4$ & $56.7 \pm 1.8$ & $17.3 \pm 1.5$ & $67.3 \pm 1.8$ \\

\hline

SC & $71.3 \pm 2.6$ & $80.2 \pm 0.9$ & $85.8 \pm 1.9$ & $72.1 \pm 1.7$ & $59.7 \pm 1.6$ & $70.0 \pm 1.6$ & $55.6 \pm 1.6$ & $69.5 \pm 1.6$ \\

\hline

PCA\_Kmeans & $69.4 \pm 2.8$ & $79.8 \pm 0.8$ & $80.6 \pm 1.6$ & $68.2 \pm 1.8$ & $59.1 \pm 1.8$ & $60.3 \pm 1.5$ & $26.7 \pm 1.8$ & $73.1 \pm 1.9$\\

\hline

PCA\_SC & $70.3 \pm 2.4$ & $81.5 \pm 1.3$ & $86.3 \pm 1.4$ & $72.9 \pm 1.5$ & $60.6 \pm 1.9$ & $72.4 \pm 1.8$ & $38.6 \pm 1.5$ & $74.2 \pm 1.8$ \\

\hline

LDA\_Kmeans & $69.9 \pm 1.9$ & $82.1 \pm 1.4$ & $81.1 \pm 2.1$ & $68.8 \pm 1.5$ & $59.8 \pm 1.6$ & $61.1 \pm 1.9$ & $29.4 \pm 1.6$ & $75.1 \pm 1.6$ \\

\hline

LDA\_SC & $70.8 \pm 1.5$ & $81.9 \pm 0.9$ & $86.8 \pm 1.7$ & $74.1 \pm 2.0$ & $61.3 \pm 1.7$ & $73.2 \pm 1.6$ & $39.9 \pm 1.4$ & $75.4 \pm 1.7$ \\

\hline

LLC & $71.2 \pm 2.4$ & $82.5 \pm 1.7$ & $84.9 \pm 1.5$ & $77.3 \pm 1.8$ & $61.4 \pm 1.9$ & $66.2 \pm 1.6$ & $34.1 \pm 1.3$ & $67.5 \pm 1.5$ \\

\hline

SPLS & $72.4 \pm 1.7$ & $83.1 \pm 2.1$ & $87.2 \pm 1.8$ & $82.2 \pm 1.6$ & $63.6 \pm 1.8$ & $69.6 \pm 1.8$ & $41.4 \pm 1.6$ & $76.5 \pm 1.8$ \\

\hline

SSC & $\mathbf{73.2 \pm 1.4}$ & $\mathbf{84.3 \pm 1.6}$ & $\mathbf{88.6 \pm 1.5}$ & $\mathbf{84.1 \pm 1.5}$ & $\mathbf{64.1 \pm 1.4}$ & $\mathbf{72.2 \pm 1.4}$ & $\mathbf{46.8 \pm 1.3}$ & $\mathbf{79.8 \pm 1.6}$ \\

\hline

\end{tabular}

\label{NMI}

\end{table*}

\subsection{Experimental Results}

The experimental results on listed in Table \ref{ACC} and Table \ref{NMI}. We can see from the two tables that our method is consistently the best algorithm using both evaluation metrics.
We also observe that:
\begin{enumerate}
\item The spectral clustering algorithm and its variants achieve better performance than the classical $k$-means and its variants. This observation suggests that it is beneficial to utilize the pairwise similarities between all data points from a weighted graph adjacency matrix that contains helpful information for clustering.
\item PCA\_Kmeans and LDA\_Kmeans are better than K-means whereas PCA\_SC and LDA\_SC are better than SC. This demonstrates that dimensionality reduction is helpful for improving the cluster performance.
\item LDA\_Kmeans outperforms PCA\_Kmeans while LDA\_SC outperforms PCA\_SC. This indicates that LDA is more capable of keeping the structural information than PCA when doing dimensionality reduction.
\item Among various spectral clustering variants, LLC is the most robust algorithm. This means using a more sophisticated graph Laplacian is beneficial for better exploitation of manifold structure.
\item SPLS is the second best clustering algorithm. This is because it incorporates both the linear and nonlinear structures of original data.
\item Our proposed Spectral Shrunk Clustering (SSC) consistently outperforms the other K-means based and spectral clustering based algorithms. This advantage is attributed to the optimal manifold learning in the low-dimensional subspace and it being tightly coupled with the clustering optimization.
\end{enumerate}

\subsection{Parameter Sensitivity}
In this section, we study the sensitivity of our algorithm \emph{w.r.t.} the parameter $\gamma$ in Eq. \eqref{objsc}. Fig \ref{ParameterSensitivity} shows the accuracy ($y$-axis) of SSC for different $\gamma$ values ($x$-axis) on all the experimental datasets. It can be seen from the figure that the performance varies when different values of $\gamma$ are used. However, except on MSRA50 and USPS datasets, our method attains the best/respectable performance when $\gamma=1$. This indicates that our method has a consistent preference on parameter setting, which makes it uncomplicated to get optimal parameter value in practice.

\begin{figure*}[!ht]
\centering
\subfigure[AR]{
\includegraphics[scale=0.0735]{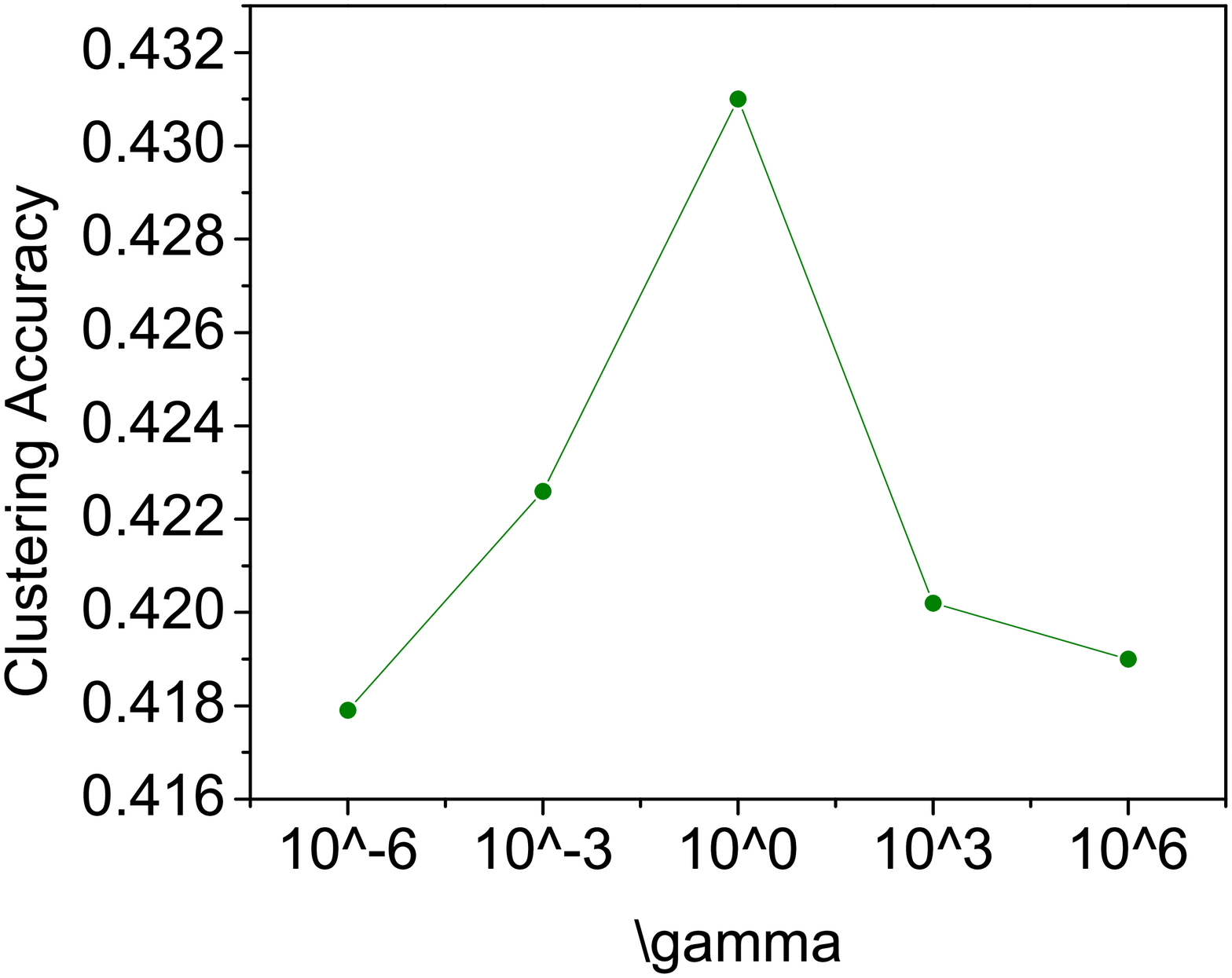}}
\subfigure[JAFFE]{
\includegraphics[scale=0.0735]{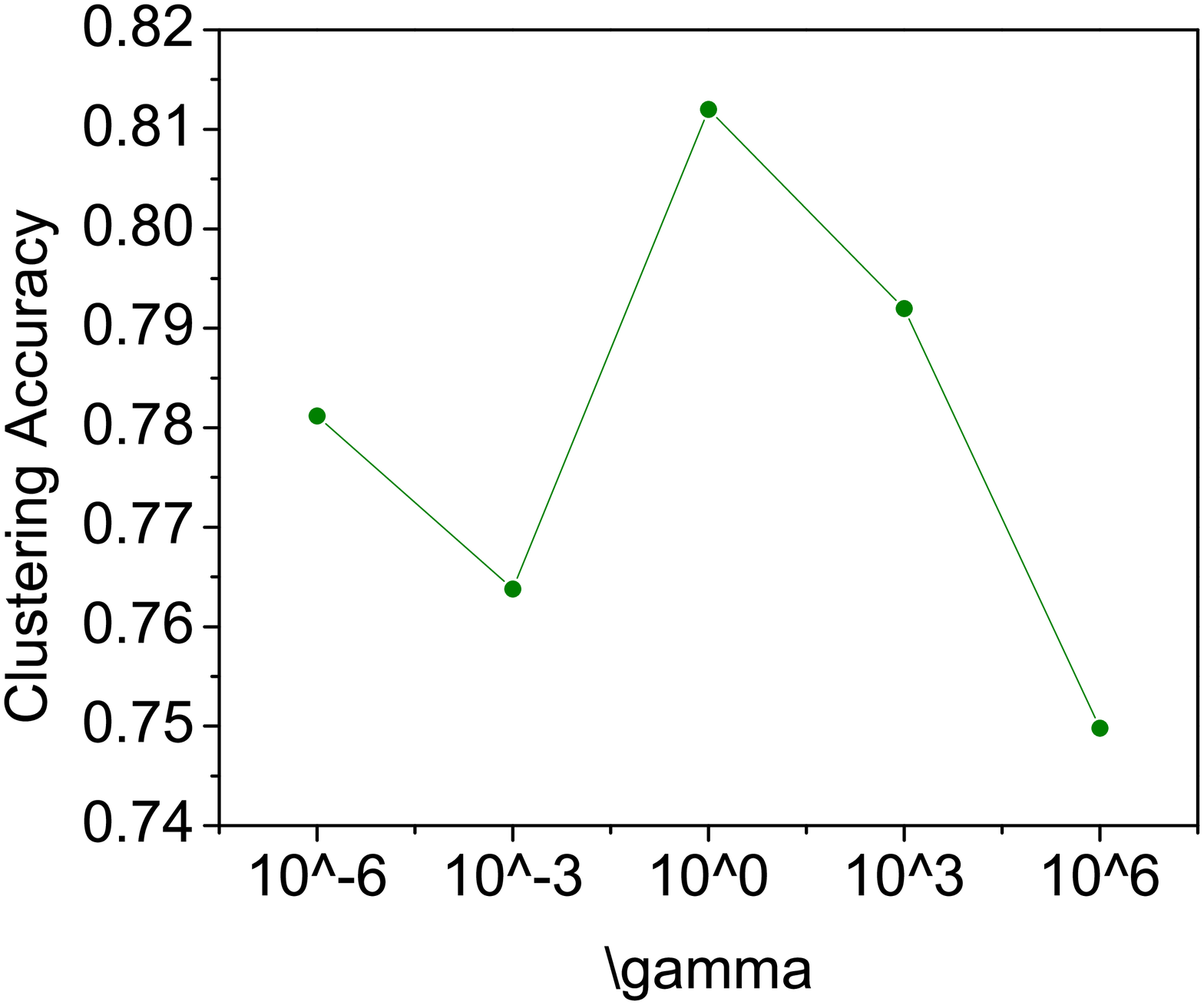}}
\subfigure[ORL]{
\includegraphics[scale=0.0735]{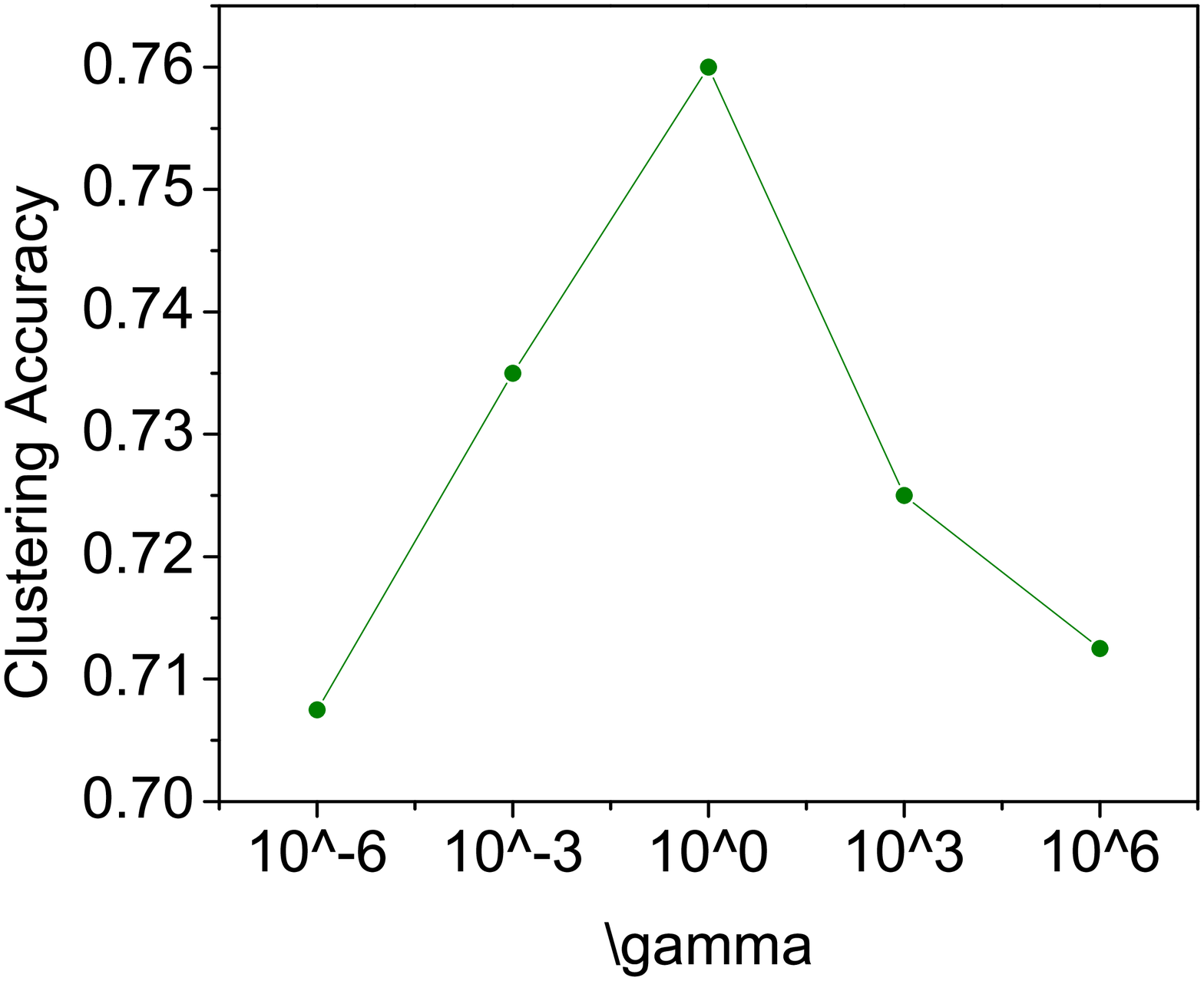}}
\subfigure[UMIST]{
\includegraphics[scale=0.0735]{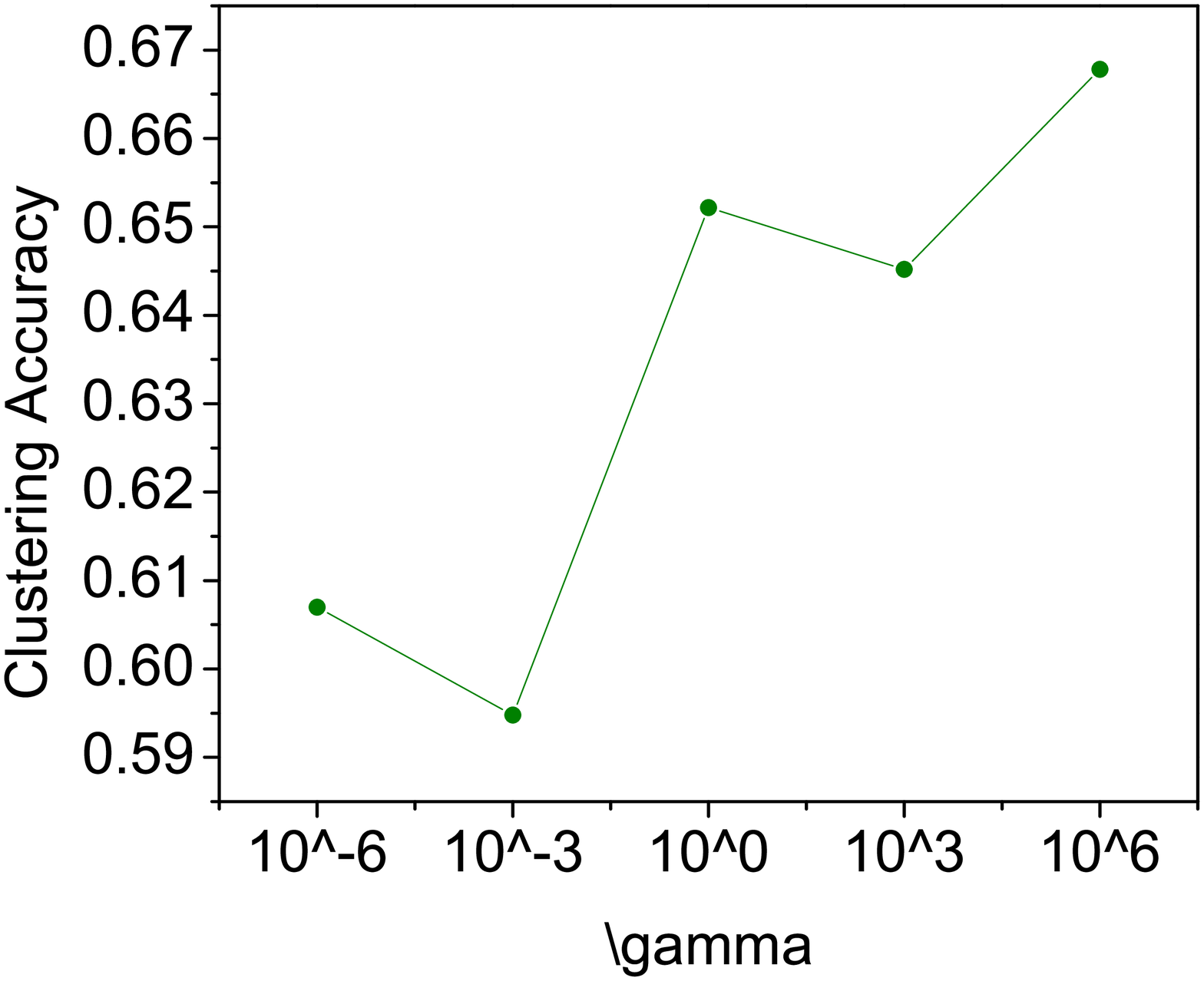}}
\subfigure[binalpha]{
\includegraphics[scale=0.0735]{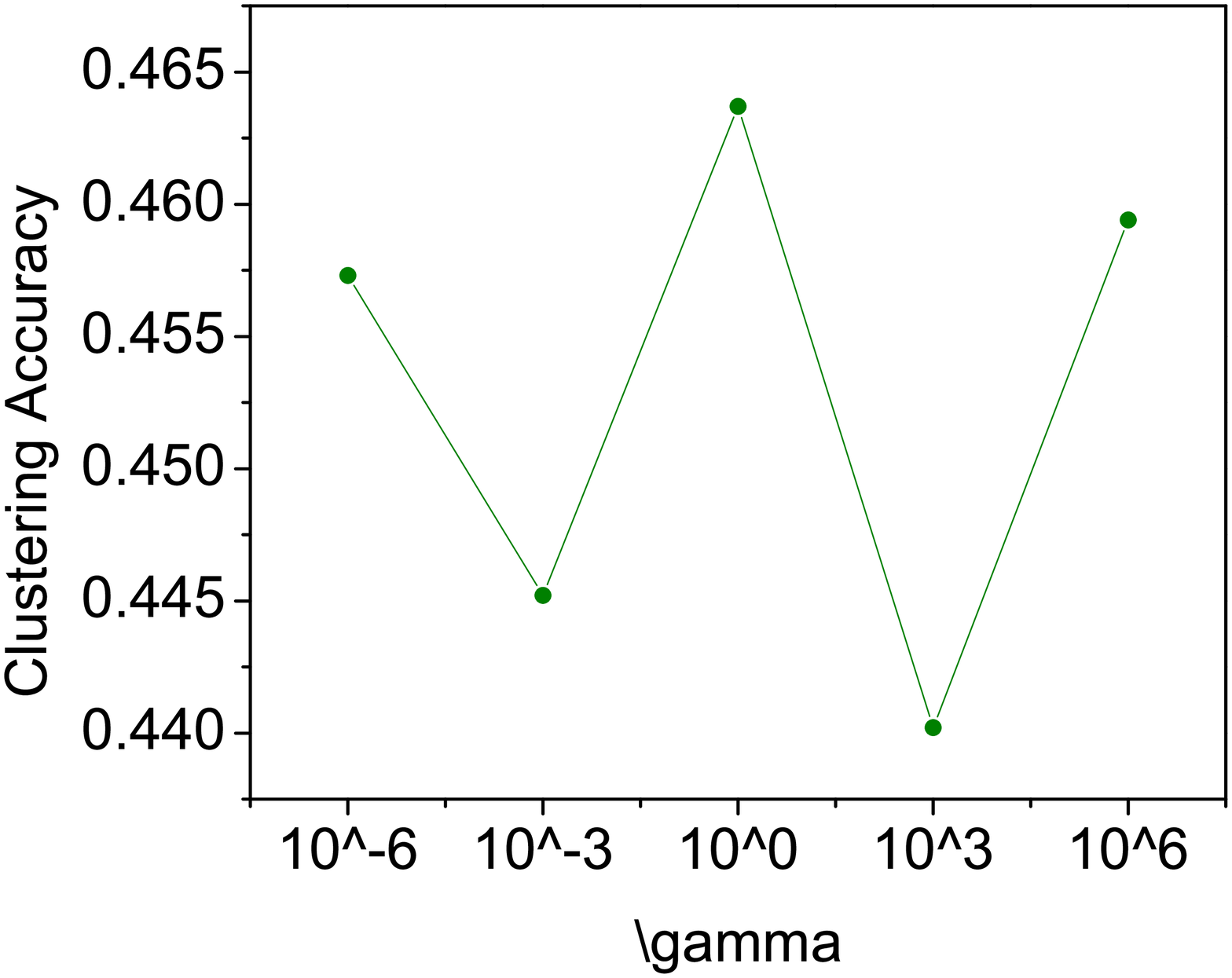}}
\subfigure[MSRA50]{
\includegraphics[scale=0.0735]{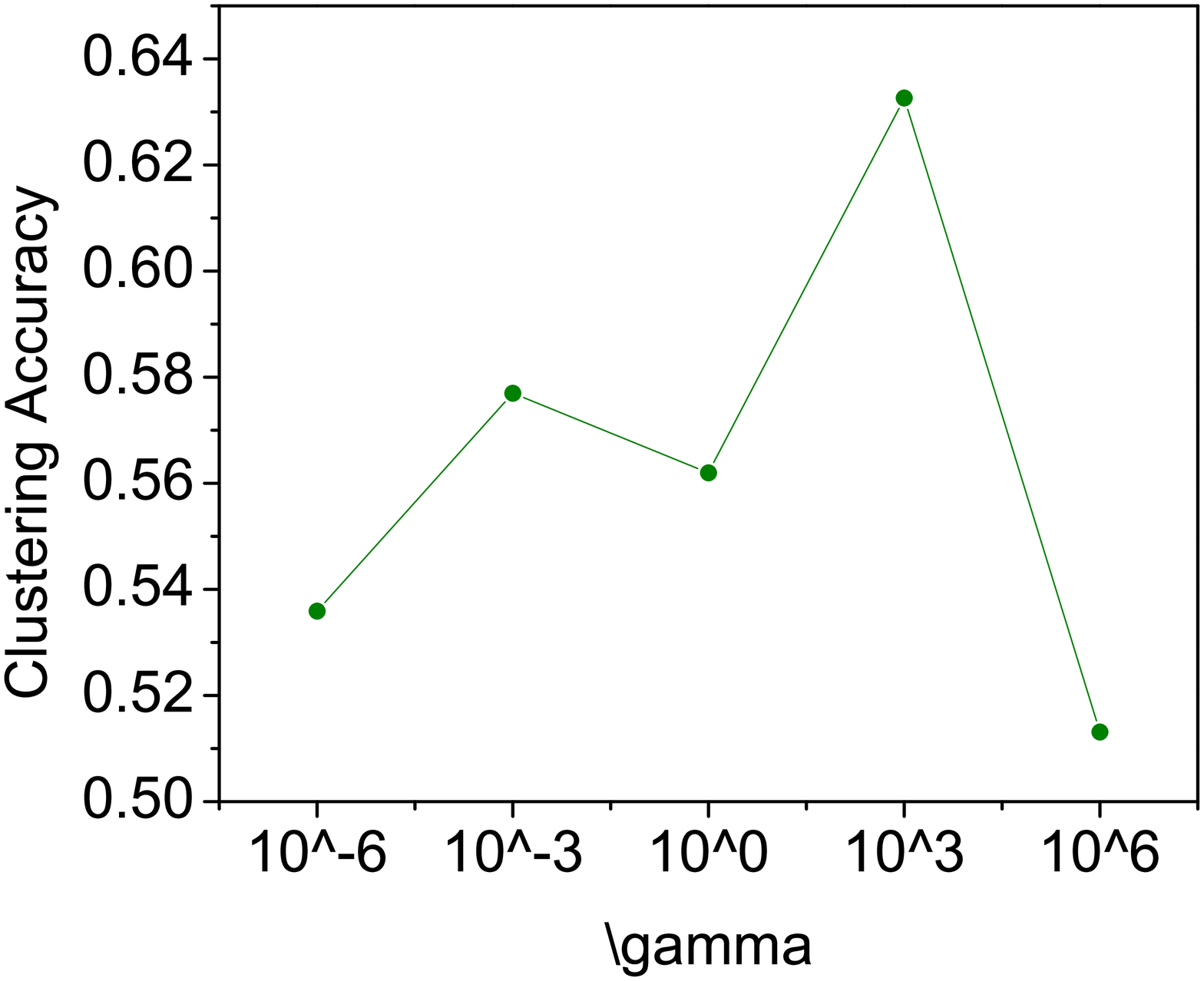}}
\subfigure[YaleB]{
\includegraphics[scale=0.0735]{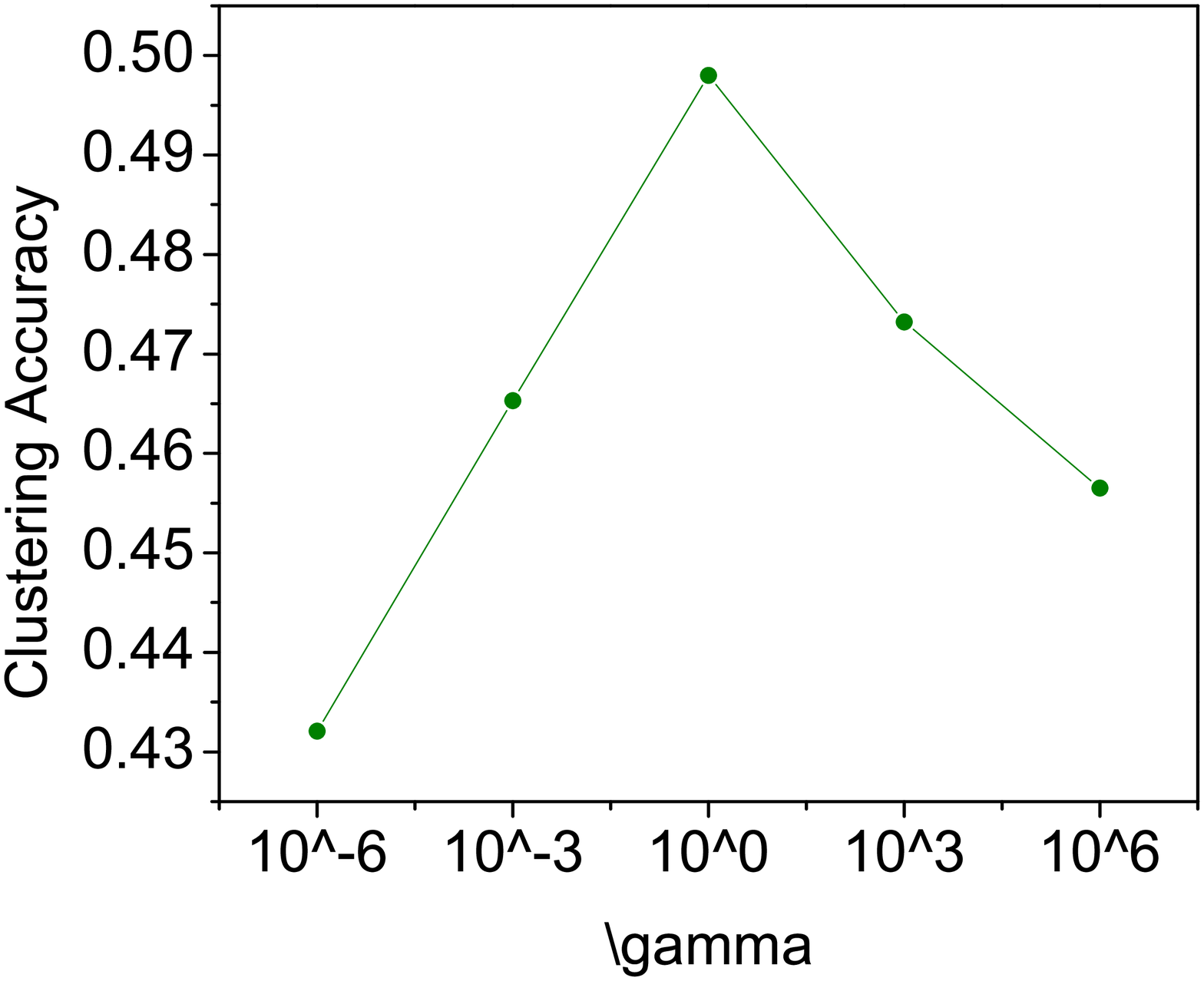}}
\subfigure[USPS]{
\includegraphics[scale=0.0735]{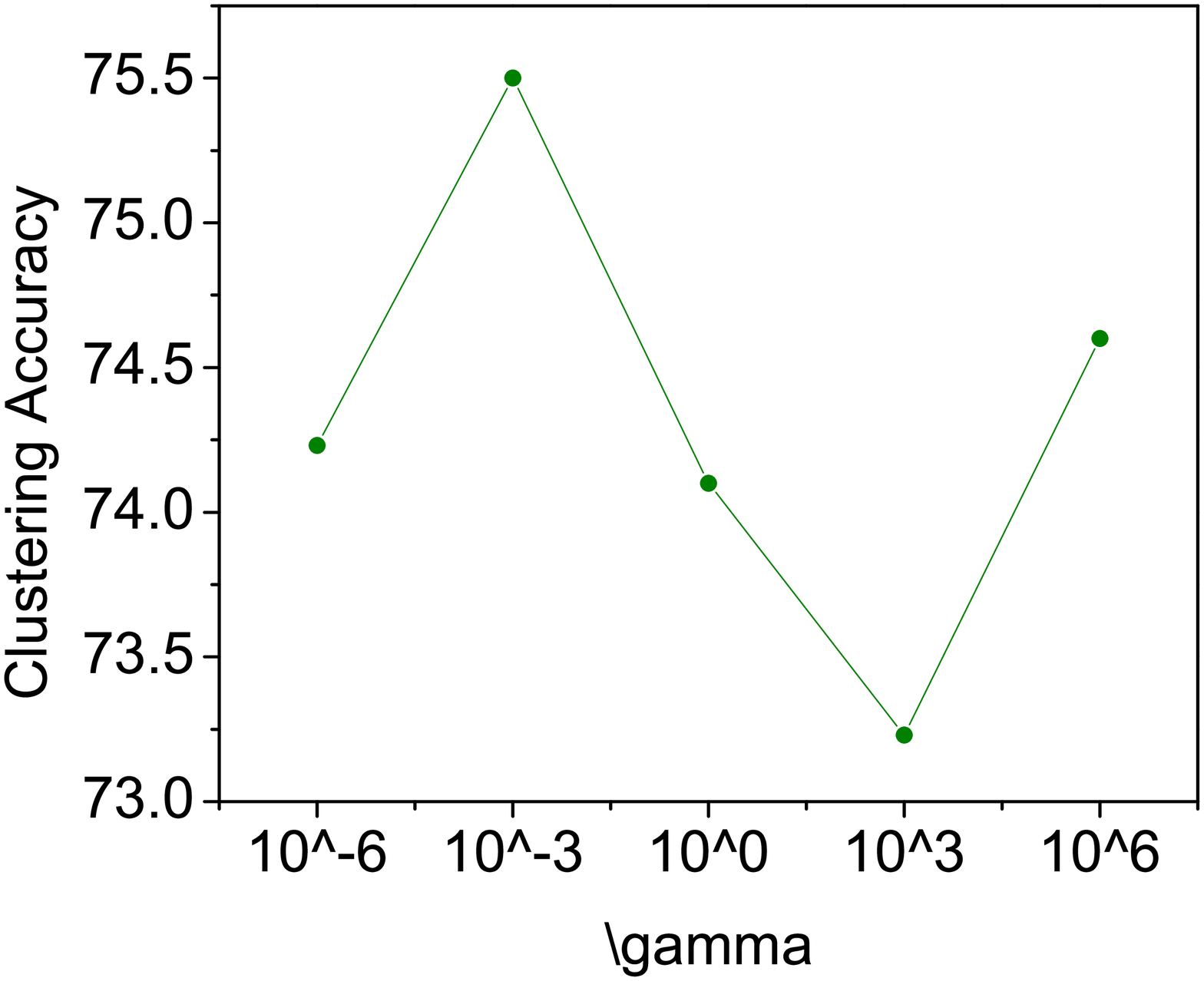}}
\caption{The clustering performance (ACC) variation of our algorithm \emph{w.r.t.} different parameter settings. From the experimental results, we observe that the proposed algorithm has a consistent preference on parameter setting, which makes it uncomplicated to get optimal parameter value in practice.}
\label{ParameterSensitivity}
\end{figure*}

\subsection{Convergence Study}
As mentioned before, the proposed iterative approach in {Algorithm}~\ref{alg:1} monotonically decreases the objective function value in Eq. (\ref{objsc}). In this experiment, we show the convergence curves of the iterative approach on different datasets in Figure \ref{converge}. The parameter $\gamma$ is fixed at 1, which is the median value of the tuned range of the parameters.

It can be observed that the objective function value converges quickly. The convergence experiment demonstrates the efficiency of our algorithm.
\begin{figure*}[!ht]
\centering
\subfigure[AR]{
\includegraphics[scale=0.0735]{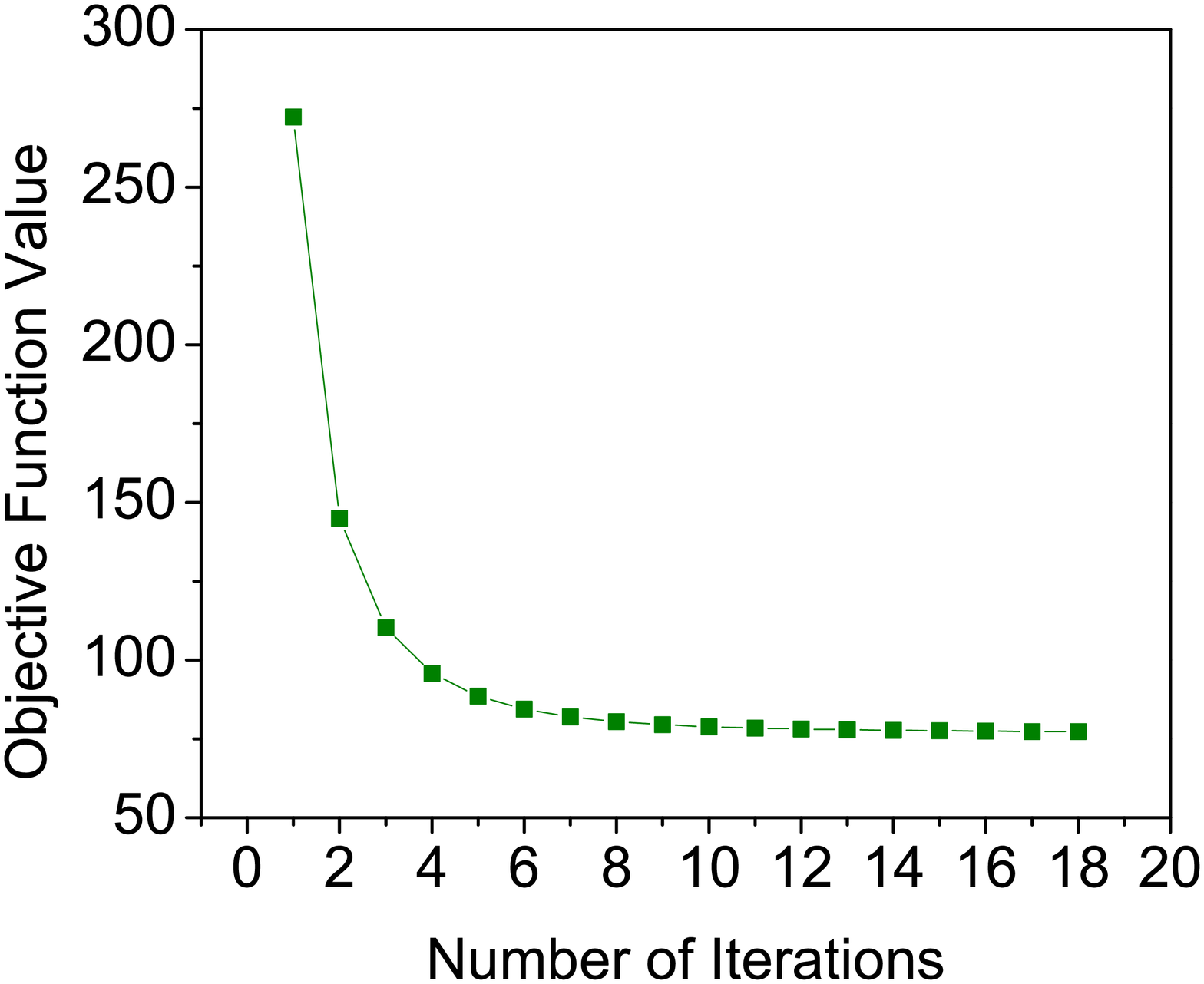}}
\subfigure[JAFFE]{
\includegraphics[scale=0.0735]{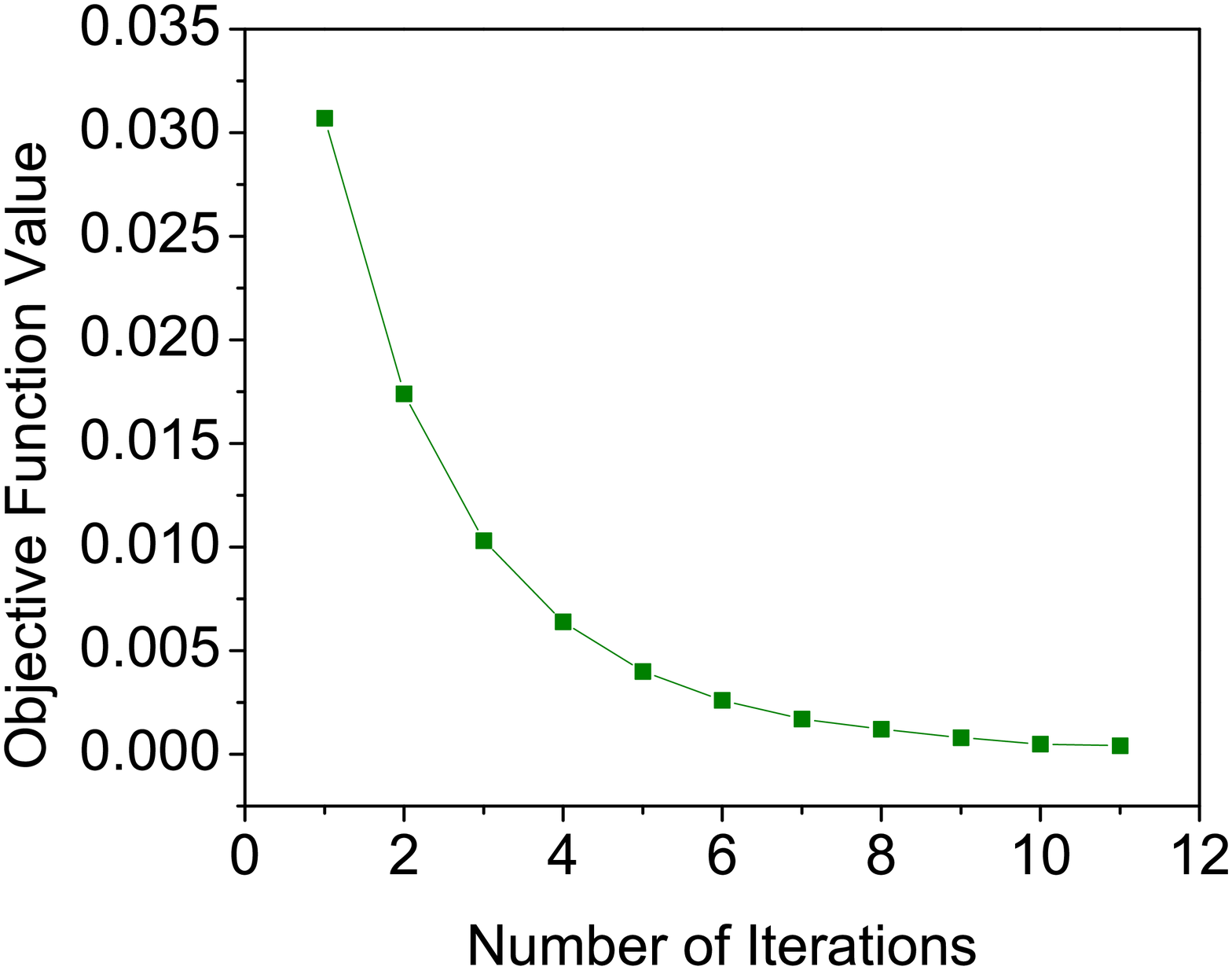}}
\subfigure[ORL]{
\includegraphics[scale=0.0735]{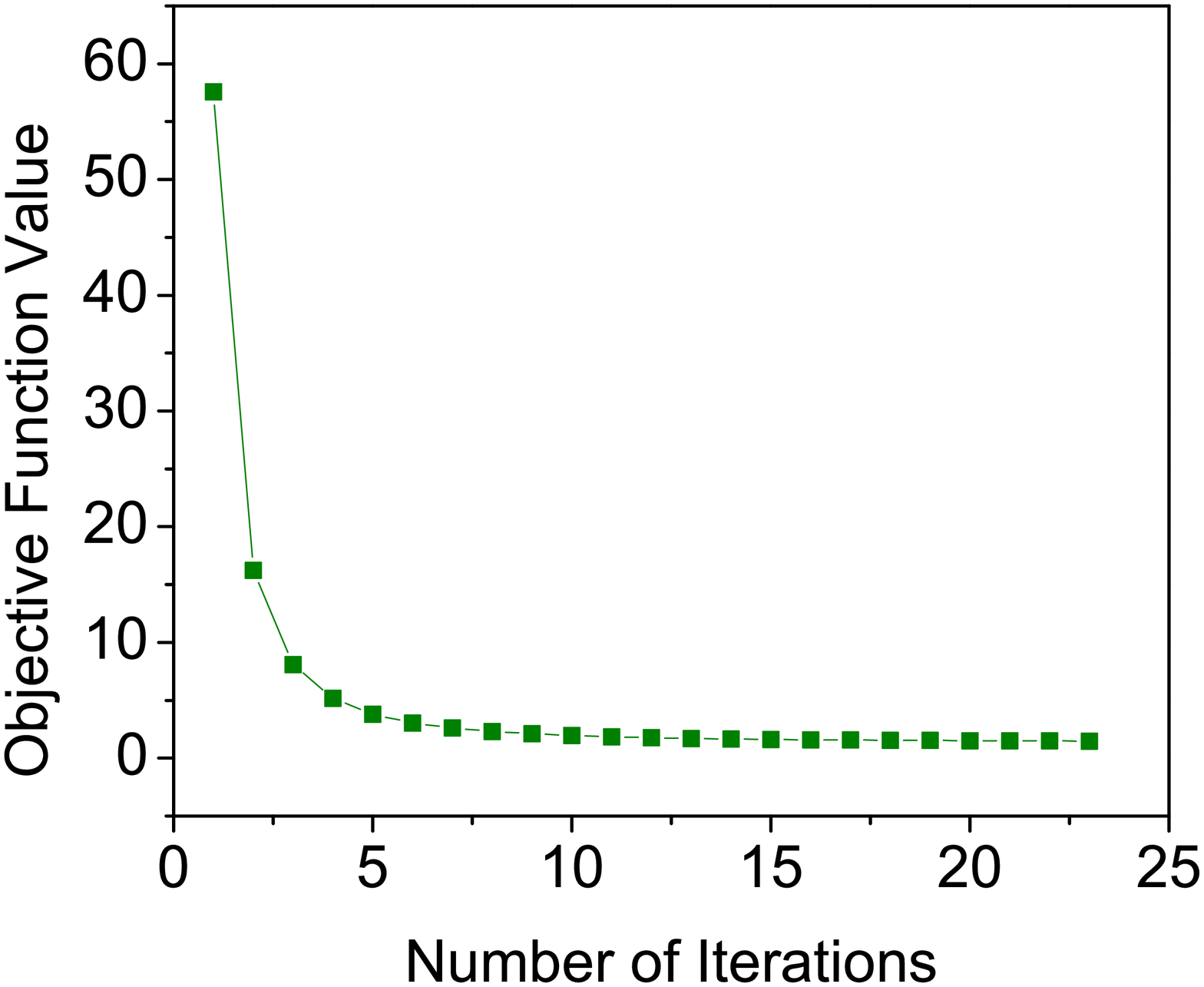}}
\subfigure[UMIST]{
\includegraphics[scale=0.0735]{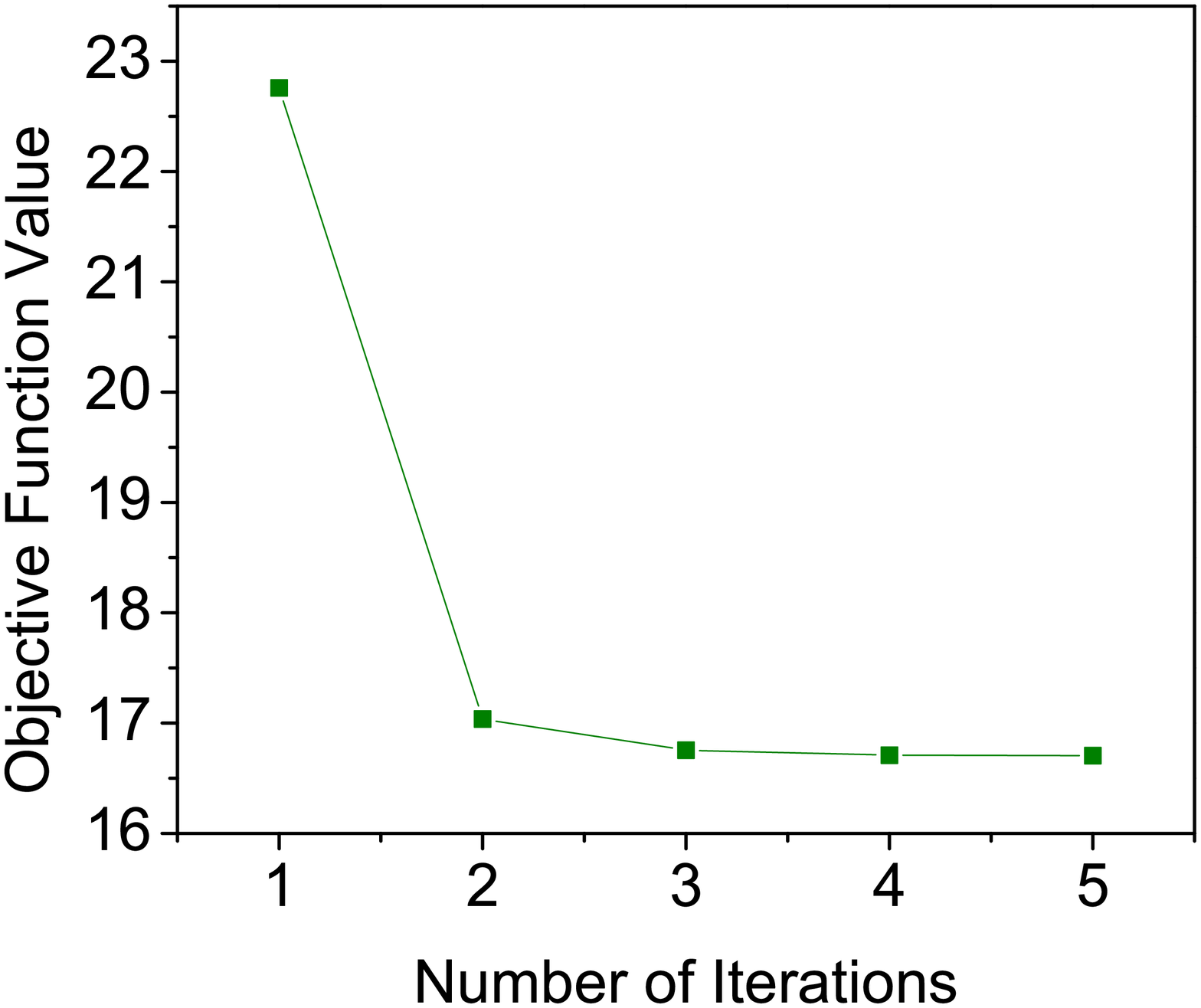}}
\subfigure[binalpha]{
\includegraphics[scale=0.0735]{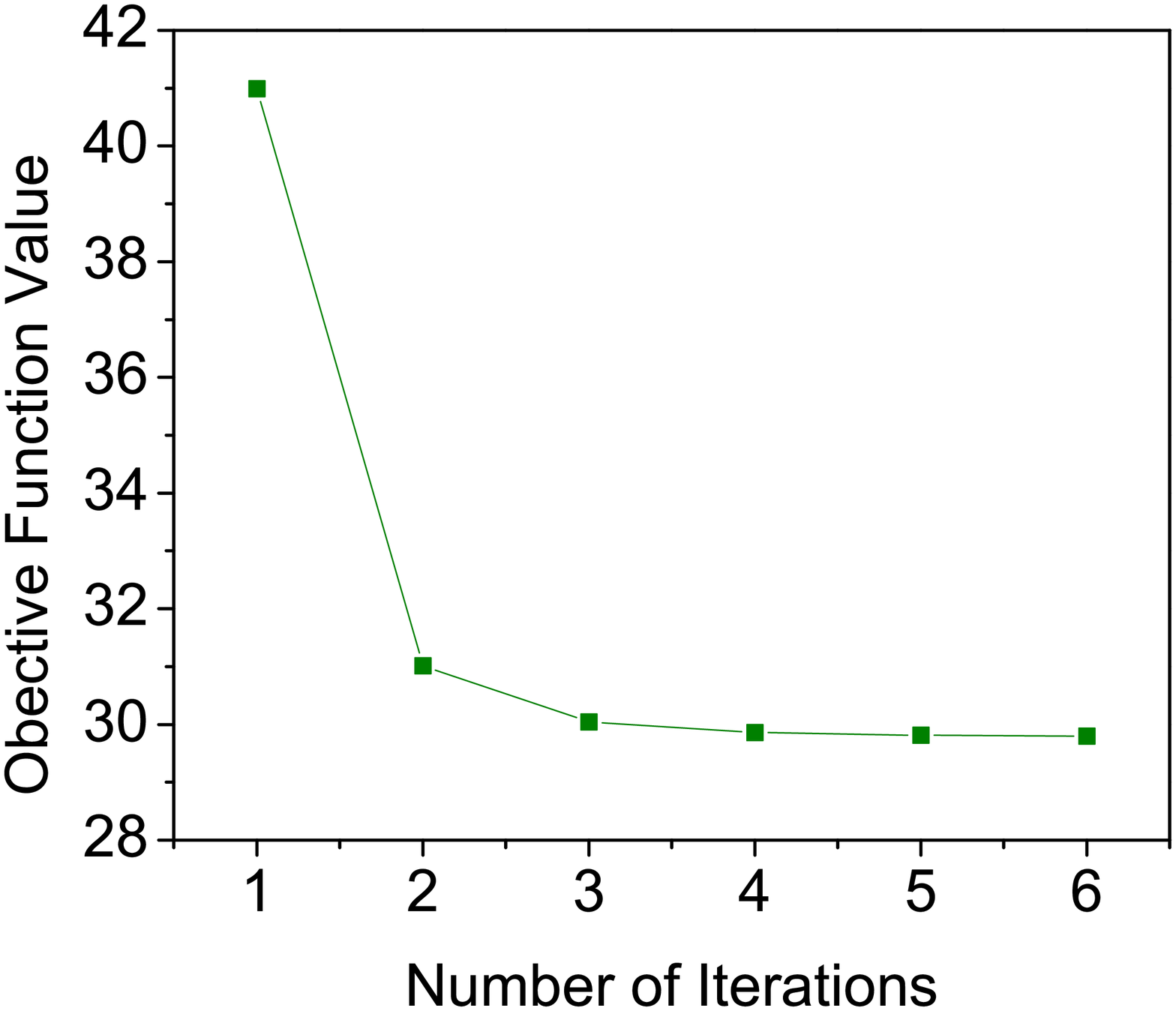}}
\subfigure[MSRA50]{
\includegraphics[scale=0.0735]{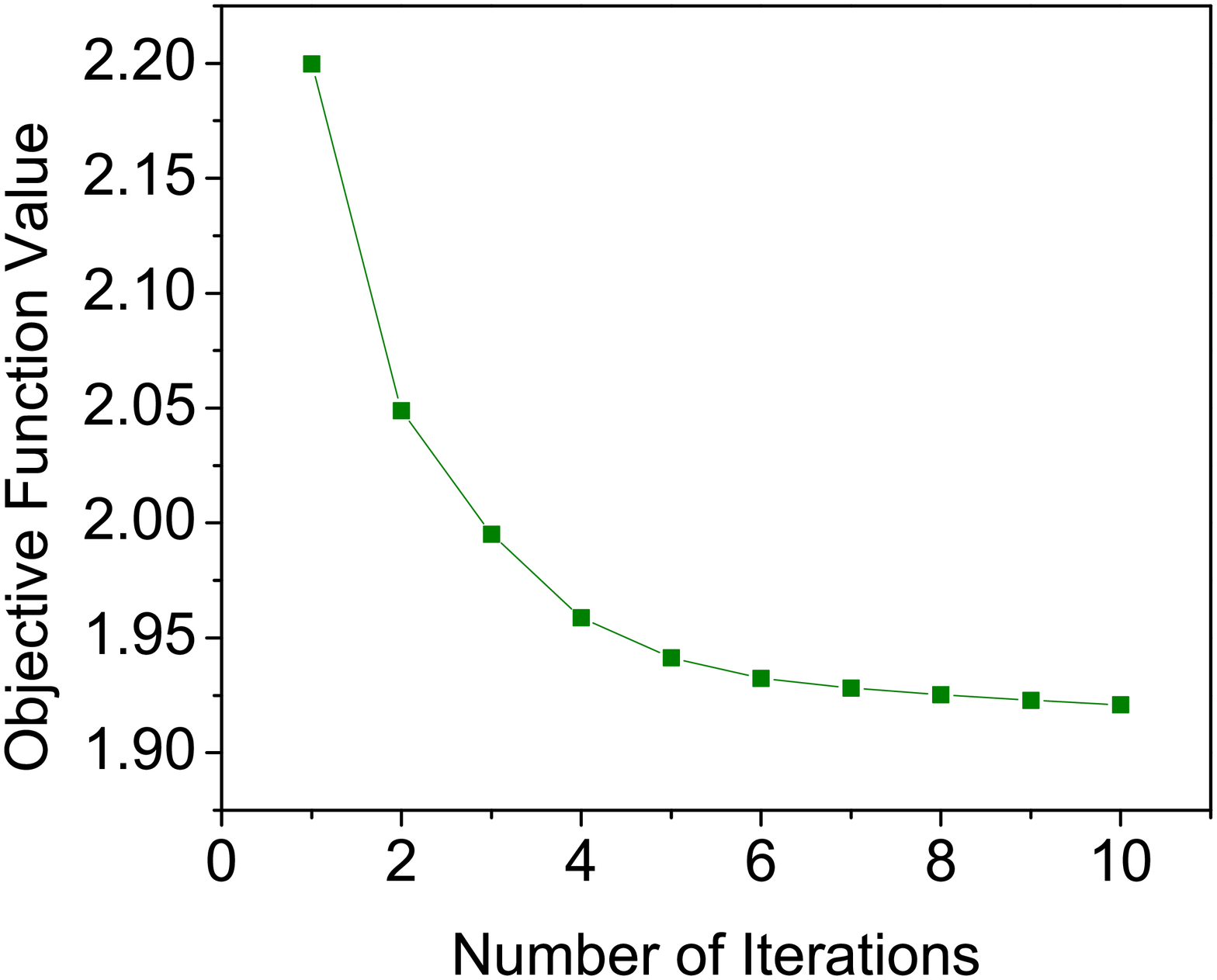}}
\subfigure[YaleB]{
\includegraphics[scale=0.0735]{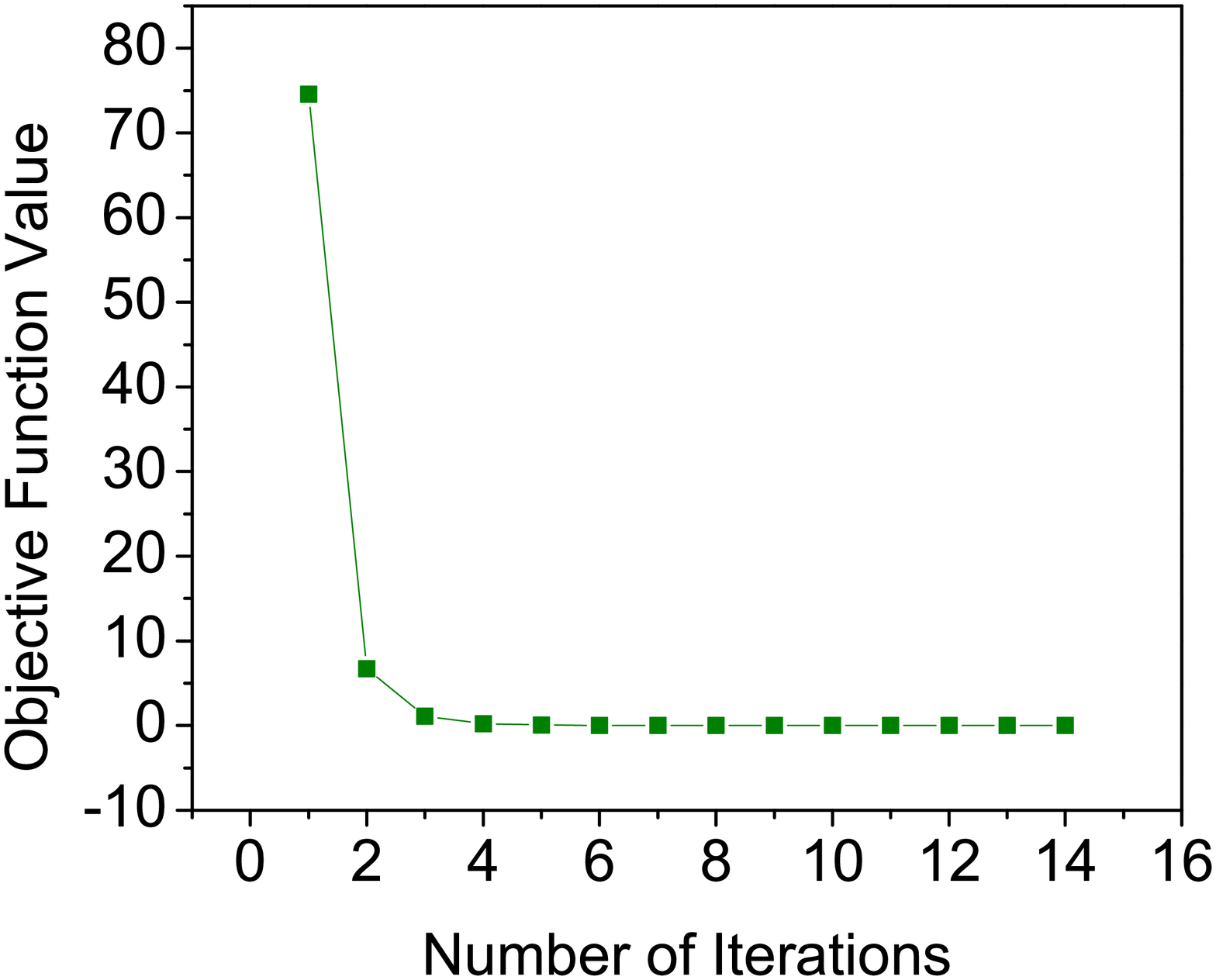}}
\subfigure[USPS]{
\includegraphics[scale=0.0735]{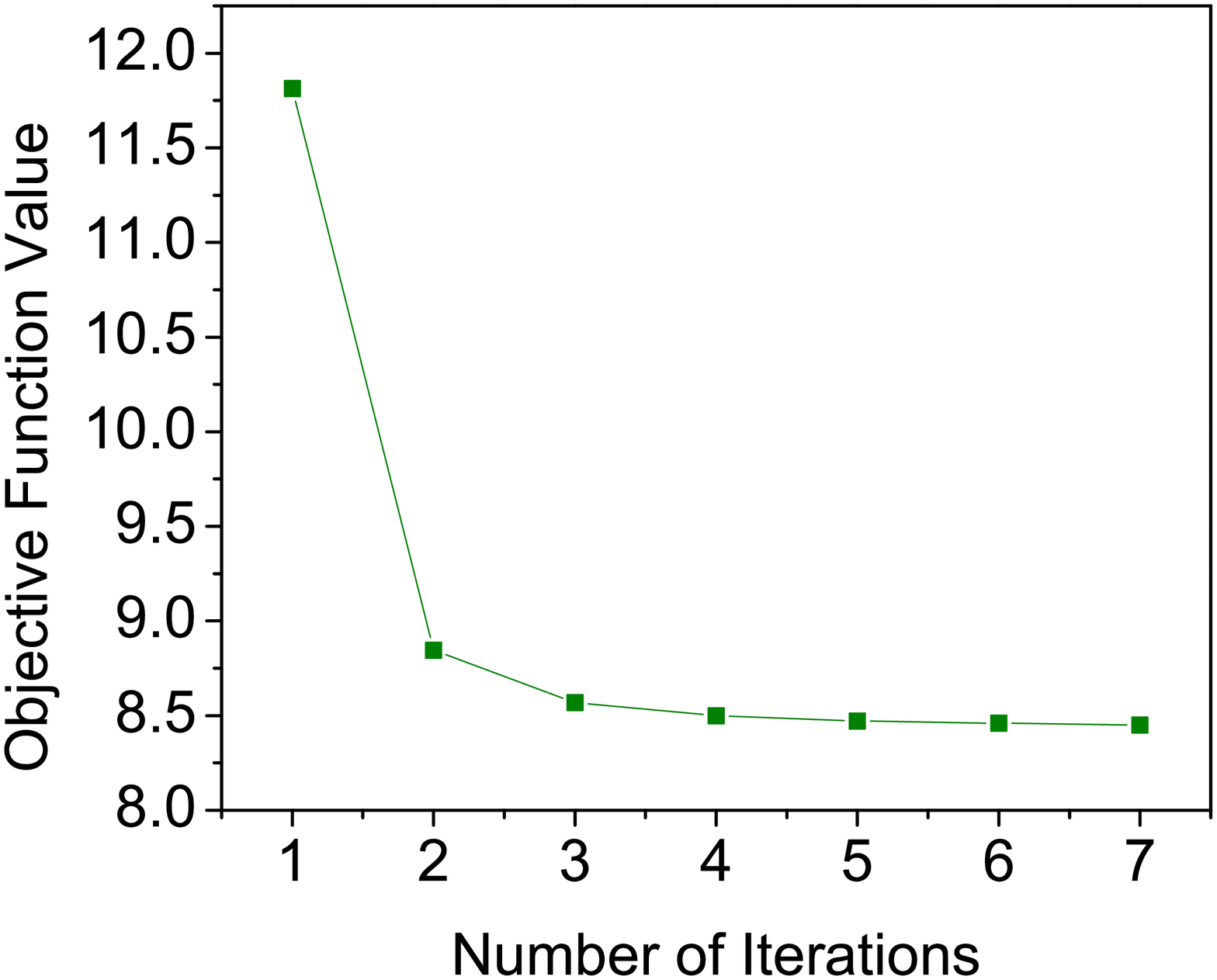}}
\caption{The convergence curves of our algorithm on different datasets. From the figures, we can observe that the objective function converges quickly, which demonstrates the efficiency of the proposed algorithm. }
\label{converge}
\end{figure*}

\section{Conclusion}
In this paper, we have proposed a novel convex formulation of spectral shrunk  clustering. The advantage of our method is three-fold. First, it is able to learn the manifold structure in the low-dimensional subspace rather than the original space. This feature contributes to more precise structural information for clustering based on the low-dimensional space. Second, our method is more capable of uncovering the manifold structure. Particularly, the shrunk pattern learned by the proposed algorithm does not have the orthogonal constraint, which makes it more flexible to fit the manifold structure. The learned manifold knowledge is particularly helpful for achieving better clustering result. Third, our algorithm is convex, which makes it easy to implement and very suitable for real-world applications. Extensive experiments on a variety of applications are given to show the effectiveness of the proposed algorithm. By comparing it to several state-of-the-art clustering approaches, we validate the advantage of our method.

{
\begin{small}
\bibliographystyle{aaai}
\bibliography{sc}
\end{small}

}

\end{document}